\documentclass[10pt]{article}

\usepackage[utf8]{inputenc} 
\usepackage[T1]{fontenc}    
\usepackage{hyperref}       
\usepackage{url}            
\usepackage{booktabs}       
\usepackage{amsfonts}       
\usepackage{nicefrac}       
\usepackage{microtype}      
\usepackage{xcolor}         
\usepackage{adjustbox}
\usepackage{amsmath,amssymb,amsfonts}
\usepackage{amsthm}
\newtheorem{proposition}{Proposition}

\newtheorem{lemma}{Lemma}
\newtheorem{remark}{Remark}
\usepackage{graphicx}
\usepackage{booktabs}
\usepackage{algorithm}
\usepackage{algpseudocode}
\usepackage{hyperref}
\usepackage{natbib}
\usepackage{caption}
\usepackage{times}
\usepackage{tikz}
\usepackage{quantikz}
\usepackage{verbatim}
\usetikzlibrary{arrows.meta,positioning,fit,backgrounds,shapes.multipart}
\usepackage{subcaption}
\usepackage{braket}

\def \f{\mathrm{f}}
\def \n{\mathrm{n}}

\newcommand{\be}{\begin{equation}}
\newcommand{\ee}{\end{equation}}
\newcommand{\beano}{\begin{eqnarray*}}
\newcommand{\eeano}{\end{eqnarray*}}

\def\bmatrix#1{\left[ \begin{matrix} #1 \end{matrix} \right]}

\title{A Cognitive-Aware  QML-CRL Framework for Detecting Affinity and Romance-Investment Fraud}

\date{}
\author{Bibhas Adhikari\thanks{Corresponding author, Fujitsu Research of America, Inc., Santa Clara, CA, USA, Email: badhikari@fujitsu.com } \,\, and \,\,   Ramya Srinivasan\thanks{Security Science Lab, Fujitsu Research of America, Santa Clara, CA 95054 Inc. Email: ramya@fujitsu.com } }

\begin{document}

\maketitle

\begin{abstract}
We present a hybrid quantum-classical framework that detects affinity and
romance-investment fraud by modelling the cognitive biases in a manipulative
conversation. In our proposed framework, cognitive biases central to this fraud class are carried by dedicated
qubits in a structured parameterized quantum circuit, together
with a frame qubit makes the encoding
sensitive to the temporal order of manipulative reframing, and a narrative
qubit that aggregates co-occurrence through a trainable entanglement
layer. The circuit parameters are trained
jointly with a classical reinforcement-learning agent that decides,
turn by turn, whether to flag the conversation, modeled as an optimal
stopping problem. We evaluate the model's performance on synthetic conversations that include
hard negatives, legitimate but urgent, and legitimate but pushy sales
conversations.
\end{abstract}

\vspace{6pt}

\section{Introduction}

Affinity and romance-investment fraud exploit trust built through shared
identity or emotional intimacy to bypass the scrutiny that a person would
normally apply to a financial decision. In affinity fraud, perpetrators
target members of an identifiable group such as a religious congregation,
ethnic community, or professional network. This helps the perpetrators to build in-group trust and the
social vouching reduce due diligence. The U.S.A. Securities and Exchange
Commission has long flagged this as a recurring driver of large-scale
Ponzi schemes \cite{securities2013affinity}. On the other hand, romance-investment fraud follows a closely related but
distinct mechanism, documented in detail by Whitty and Buchanan \cite{whitty2013}. This fraud is based on the scam's
staged structure: relationship-building, a small demonstration of
``legitimacy'', and then escalating financial requests. It is also investigated in the literature why victims
continue sending money even after clear warning signs appear, attributing
it to the relationship narrative overriding financial judgment rather than
any deficit in the victim's general reasoning ability \cite{buchananwhitty2014}.
In  the UK Office of Fair Trading report, this pattern appears across fraud types,
showing that authority, urgency, and social-proof cues systematically
override rational risk assessment once a persuasive relationship or
narrative has been established \cite{leafischerevans2009}. Thus, affinity and romance-investment fraud succeed by exploiting cognitive
biases in a specific, staged sequence rather than by any single deceptive
claim \citep{pratkanis2005weapons,whitty2013,buchananwhitty2014,leafischerevans2009,coluccia2020online,dearden2025mechanisms}.

In their proposal of classical
Behavioral Reinforcement Learning (BRL) in \citep{buetgolfouse2025brl}, the authors show how to embed biases such as Loss Aversion through an asymmetric utility
function into a reinforcement learning framework. Indeed, the utility function in their paper represents a pointwise transform of a scalar reward which is not designed to reproduce two well established human judgment phenomena: \textit{order effects} 
\citep{tversky1974,busemeyer2012} and the \textit{conjunction fallacy} \citep{tversky1983conjunction,busemeyer2011conjunction}. These two notions play a key role into the development of quantum cognition model \cite{busemeyer2015quantum,widdows2023quantum}, and
both are central to how affinity and romance-investment fraud actually
work
\citep{pratkanis2005weapons}. The role of cognitive biases in financial decision making is explored in the literature in different contexts, for instance, see \cite{ruggeri2023persistence,de2024financial,ruggeri2024financial,hammad2026cognitive,tkacz2025exploration,gabhane2023behavioral}. There are also studies to employ cognitive biases into define feature vectors for learning tasks \cite{cardie2000cognitive,kaleem2017cognitive}.
Quantum-cognitive models of decision-making \citep{busemeyer2012,khrennikova2014order} address
both by representing beliefs as vectors in a Hilbert space and questions
as (generally non-commuting) projective measurements. 

In this paper, we consider the question: \textit{given a conversation, can a
system watching it learn to detect that this kind of manipulation is under
way, and if so, when?} We investigate it with the use of a structured parameterized
quantum circuit (PQC), often refereed to as Quantum Neural Network or ansatz \cite{abbas2021power} in the context of Quantum Machine Learning (QML), acting as a trainable cognitive encoder  followed by a
Classical Reinforcement Learning (CRL) stopping agent. In our framework, we consider a quantum register of $n$ qubits to encode $n$ cognitive biases driven by cue-conditioned rotation modeled by rotation parametrized gates and two additional qubits: a \textit{frame} qubit that carries the temporal
order of manipulative reframing through non-commuting rotations, and a
\textit{narrative} qubit that aggregates co-occurrence. Indeed, the parametrized rotation operators encode the cognitive bias cues, whose values are trained and learned with the aid of the RL agent. The proposed scam 
detector therefore does not merely use a quantum representation of
cognitive bias, it learns one, under an objective that is the
deployment objective (flag early enough, but do not cry wolf) rather than a
surrogate classification loss. Thus, our framework introduces a quantum-classical hybrid QML-CRL model for scam detection, which combines a cognitive-aware quantum feature encoder with a classical reinforcement learning agent, providing a realization of Quantum Reinforcement Learning (QRL) \cite{meyer2022survey,chen2026quantum} model for sequential scam detection.


The structure of the ansatz that we consider in this paper is motivated by the following observation in the context of the scam. 
First, manipulative reframing
is order-dependent: presenting a loss frame before a recovery frame is not
equivalent to the reverse, and a same-axis rotation cannot distinguish the
two, since $R_Y(\theta_1)R_Y(\theta_2)=R_Y(\theta_1+\theta_2)$ regardless of
application order. We address this with a frame qubit that receives $R_Z$ for
loss-framing cues and $R_Y$ for recovery-framing cues; because
$[R_Y(\theta),R_Z(\phi)] = -2i\sin(\theta/2)\sin(\phi/2)X \neq 0$
(Proposition~\ref{prop:noncommute}), the resulting state depends on the
temporal order in which the two framings occur. We show, however, that
non-commutativity of the gates is not by itself sufficient and the measurement basis must also avoid a degenerate choice at which the
order dependence is invisible (Proposition~\ref{prop:orderblind}), and we
therefore make that basis a trained parameter. Second, the biases do
not operate independently i.e. a target's judgement that a person is trustworthy
and that an opportunity is legitimate are assessed as one mutually
reinforcing narrative. Since no joint measurement can exceed the
Fr\'echet bound (Proposition~\ref{prop:frechet}), we use the narrative qubit
not to violate that bound but as an order-insensitive coherence aggregate,
implemented by controlled rotations whose angles are trained rather than
fixed. The parameters of the rotation gates are functions of cues (which are directly obtained from the conversation using an Natural Language Processing (NLP) tool for the conversation in each turn) and the cue thresholds. Besides, there are entangler parameters which are incorporated to learn the narrative coherence in the conversation. 

We emphasize that order effects and the conjunction fallacy are well documented violations of
Kolmogorov probability applied to a single fixed sample space \cite{wang2006order,trueblood2011quantum,franco2009conjunction}. This constrains probabilistic models of the judgement itself but not the classical learning
machinery in general. In fact, any recurrent or attention-based sequence model is
order-sensitive simply because its state updates do not commute, and indeed our
encoder is efficiently classically simulatable. Our claim is therefore
one of inductive bias rather than expressive necessity i.e. the quantum formalism
supplies a compact, interpretable encoding in which order sensitivity, bounded
monotone readouts, and Fr\'echet-respecting aggregation follow from the
structure of the representation rather than having to be learned from data or
imposed by regularization. 

In our numerical demonstration, we consider four cognitive biases as described below, whose existence in these frauds can be clearly justified.
\begin{itemize}
\item \textbf{Loss aversion} \citep{kahneman2013prospect} - the engine of
the funnel's ``recovery pressure'' stage: once money is sunk, framing
continuation as avoiding a certain loss (rather than risking more) triggers
risk-seeking behavior in the loss domain.
\item \textbf{Commitment (and consistency)} \citep{cialdini2001} - having
made a small commitment, a target feels pressure to act consistently with
it rather than admit it was a mistake; this compounds with loss aversion
rather than acting as an independent factor.
\item \textbf{Reciprocity} \citep{cialdini2001} - a real small payout (or,
in affinity fraud, a community member's personal vouching) creates a felt
obligation to reciprocate trust, directly enabling the next, larger ask.
\item \textbf{Authority} \citep{cialdini2001} - borrowed authority (a
trusted community figure in affinity fraud; fabricated credentials in
romance-investment fraud) is a primary trust-conferral mechanism
\citep{perribrody2011}.
\end{itemize}

The CRL framework is then designed to convert the quantum encoder's continuous, per-turn risk signal into a single actionable decision: when, if ever, to flag a conversation. The RL agent  learns from labeled outcomes to trade off the cost of a false alarm against the cost of intervening too late, rather than simply classifying the conversation after the fact. Thus, we frame the detection as an \textit{optimal stopping problem}. The RL layer receives no information about the conversation except the quantum register's readout. We perform numerical test on synthetic, labelled conversation
data: order-effect magnitude at both the feature level and the detection-timing level, and an ablation study isolating each encoder channel's contribution
to detection quality. 

The remainder of this paper is organized as follows. Section~\ref{sec:encoder} introduces the proposed cognitive-aware PQC for modeling manipulative conversations. Section~\ref{sec:rl} presents the joint QML-CRL framework for scam detection. Section~\ref{sec:testing} reports numerical results on synthetic data. Finally, the paper concludes with a discussion on limitations and future work in Section \ref{sec:limits}, followed by supplementary material after the references.

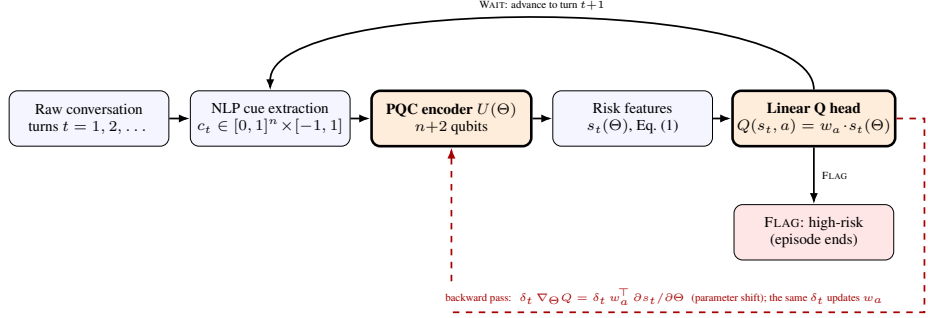
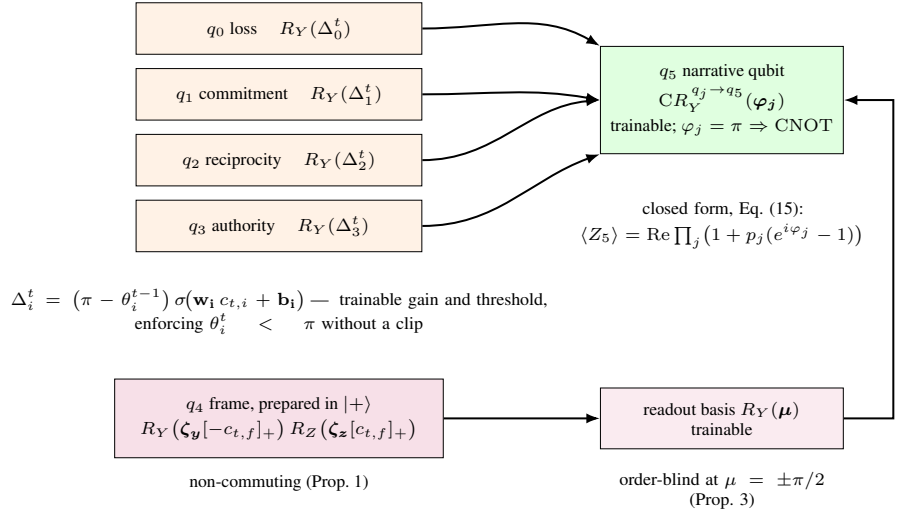
\begin{figure*}[h]
\centering
\begin{subfigure}[b]{\linewidth}
\centering
\adjustbox{max width=\linewidth}{%
\begin{tikzpicture}[
  box/.style={draw, rounded corners, minimum height=0.9cm,
              minimum width=2.6cm, align=center, font=\scriptsize,
              fill=blue!4, inner sep=2pt},
  tbox/.style={box, fill=orange!14, very thick},
  arr/.style={-{Latex[length=1.8mm]}, thick},
  grad/.style={-{Latex[length=1.8mm]}, thick, dashed, red!65!black}
]
\node[box]  (text) at (0,0)    {Raw conversation\\turns $t=1,2,\dots$};
\node[box]  (nlp)  at (2.95,0) {NLP cue extraction\\$c_t\in[0,1]^{n}\!\times\![-1,1]$};
\node[tbox] (enc)  at (5.90,0) {\textbf{PQC encoder} $U(\Theta)$\\
                                $n{+}2$ qubits};
\node[box]  (feat) at (8.85,0) {Risk features\\$s_t(\Theta)$, Eq.~\eqref{eq:state}};
\node[tbox] (rl)   at (11.8,0) {\textbf{Linear Q head}\\$Q(s_t,a)=w_a\!\cdot\! s_t(\Theta)$};
\node[box, fill=red!10] (flag) at (11.8,-1.85) {\textsc{Flag}: high-risk\\(episode ends)};

\draw[arr] (text) -- (nlp);
\draw[arr] (nlp)  -- (enc);
\draw[arr] (enc)  -- (feat);
\draw[arr] (feat) -- (rl);
\draw[arr] (rl) -- node[right, font=\tiny]{\textsc{Flag}} (flag);
\draw[arr] (rl.north) to[out=90,in=90,looseness=0.45]
  node[above, font=\tiny]{\textsc{Wait}: advance to turn $t{+}1$} (nlp.north);
\draw[grad] (rl.east) -- (13.6,0) -- (13.6,-3.15) -- (5.90,-3.15)
  node[above, pos=0.55, font=\tiny, red!65!black, fill=white,
       inner sep=2pt, yshift=1pt]
  {backward pass: $\;\delta_t\,\nabla_\Theta Q
    =\delta_t\,w_a^{\!\top}\,\partial s_t/\partial\Theta\;$ (parameter shift);
   the same $\delta_t$ updates $w_a$} -- (enc.south);
\end{tikzpicture}%
}
\caption{End-to-end pipeline. Solid: inference. Dashed: the
temporal-difference error, which trains the Q head and flows back
into the circuit parameters $\Theta$.}
\label{fig:architecture}
\end{subfigure}

\vspace{6pt}
\begin{subfigure}[b]{\linewidth}
\centering
\adjustbox{max width=\linewidth}{%
\begin{tikzpicture}[
  qbox/.style={draw, minimum height=0.72cm, minimum width=4.0cm,
               align=center, font=\scriptsize, fill=orange!10, inner sep=2pt},
  fbox/.style={draw, minimum height=1.1cm, minimum width=4.6cm,
               align=center, font=\scriptsize, fill=purple!12, inner sep=2pt},
  ancbox/.style={draw, minimum height=1.5cm, minimum width=3.4cm,
                 align=center, font=\scriptsize, fill=green!12, inner sep=2pt},
  robox/.style={draw, minimum height=0.85cm, minimum width=3.4cm,
                align=center, font=\scriptsize, fill=purple!8, inner sep=2pt},
  note/.style={font=\scriptsize, align=center},
  arr/.style={-{Latex[length=1.9mm]}, thick}
]
\node[qbox] (q0) at (0,0)     {$q_0$ loss \quad $R_Y(\Delta_0^t)$};
\node[qbox] (q1) at (0,-0.92) {$q_1$ commitment \quad $R_Y(\Delta_1^t)$};
\node[qbox] (q2) at (0,-1.84) {$q_2$ reciprocity \quad $R_Y(\Delta_2^t)$};
\node[qbox] (q3) at (0,-2.76) {$q_3$ authority \quad $R_Y(\Delta_3^t)$};
\node[note, text width=8.2cm] (dnote) at (0,-3.95)
   {$\Delta_i^t=\bigl(\pi-\theta_i^{t-1}\bigr)\,
     \sigma\!\bigl(\mathbf{w_i}\,c_{t,i}+\mathbf{b_i}\bigr)$\;---\;
    trainable gain and threshold,\\ enforcing $\theta_i^t<\pi$ without a clip};

\node[fbox] (q4) at (0,-5.45) {$q_4$ frame, prepared in $\ket{+}$\\[2pt]
   $R_Y\bigl(\boldsymbol{\zeta_y}[-c_{t,f}]_+\bigr)\,
    R_Z\bigl(\boldsymbol{\zeta_z}[c_{t,f}]_+\bigr)$};
\node[note] at (0,-6.35) {non-commuting (Prop.~\ref{prop:noncommute})};

\node[ancbox] (q5) at (6.2,-1.0) {$q_5$ narrative qubit\\[2pt]
   $\mathrm{C}R_Y^{\,q_j\to q_5}(\boldsymbol{\varphi_j})$\\[2pt]
   trainable; $\varphi_j=\pi\Rightarrow\mathrm{CNOT}$};
\node[note, text width=4.0cm] at (6.2,-2.75)
   {closed form, Eq.~\eqref{eq:narrclosed}:\\[1pt]
    $\langle Z_5\rangle=\mathrm{Re}\prod_j\bigl(1+p_j(e^{i\varphi_j}-1)\bigr)$};

\node[robox] (mu) at (6.2,-5.45) {readout basis $R_Y(\boldsymbol{\mu})$\\trainable};
\node[note, text width=4.0cm] at (6.2,-6.45)
   {order-blind at $\mu=\pm\pi/2$\\(Prop.~\ref{prop:orderblind})};

\draw[arr] (q0.east) to[out=0,in=150] (q5.north west);
\draw[arr] (q1.east) to[out=0,in=170] (q5.west);
\draw[arr] (q2.east) to[out=0,in=185] (q5.west);
\draw[arr] (q3.east) to[out=0,in=-160] (q5.south west);
\draw[arr] (q4.east) -- (mu.west);
\draw[arr] (mu.east) -- (8.6,-5.45) -- (8.6,-1.0) -- (q5.east);
\end{tikzpicture}%
}
\caption{Parameterized encoder detail; bold symbols are trained.}
\label{fig:encoder-detail}
\end{subfigure}
\caption{Complete architecture. (a) The end-to-end pipeline from raw
conversation text to a flag/wait decision. (b) Detail of the
parameterized quantum-cognitive encoder.}
\label{fig:architecture-full}
\end{figure*}

\section{Quantum-Cognitive PQC Encoder}
\label{sec:encoder}

In this section, we propose a PQC on an $(n+2)$-qubit quantum register that acts as a quantum-cognitive encoder. The first $n$-qubits, called the bias qubits encode $n$ cognitive biases, whereas the frame and narrative qubits are used to monitor the ordering and coherence of the conversation, respectively. The Figure \ref{fig:architecture-full} describes the end-to-end architecture of our proposal incorporating the four biases, as described in the introduction. The Algorithm \ref{alg:encode} describes the working of the PQC for the general context.

\subsection{Quantum Register}

The bias qubits are labelled $q_0,\dots,q_{n-1}$, the frame qubit is $q_n$ (also denoted as $q_\f$ to be explicit), and
the narrative qubit is $q_{n+1}$ (also denoted as $q_\n$ to be explicit), and we write $\text{pull}^t_\f$ and
$\text{pull}^t_{\mathrm{nar}}$ for their readouts, so that
$\text{pull}^t_\f=-\langle Z_n\rangle_t$ and
$\text{pull}^t_{\mathrm{nar}}=-\langle Z_{n+1}\rangle_t$. Bias and narrative
qubits are prepared in $\ket{0}$ and the frame qubit in $\ket{+}=H\ket{0}$, for
the reason given in the supplementary notes, so that $\ket{\psi(0)}=\ket{0}^{\otimes n}\otimes\ket{+}\otimes\ket{0}$, which is the initial state of the quantum register. 
 At every turn $t,$ we determine the risk feature vector 
\begin{equation}
s_t=\Bigl(\text{pull}^t_0,\dots,\text{pull}^t_{n-1},\;
\text{pull}^t_\f,\;\text{pull}^t_{\mathrm{nar}},\;
1-e^{-t/\tau_0},\;1\Bigr)
\;\in\;[-1,1]^{n+2}\times[0,1)\times\{1\}\;\subset\;\mathbb{R}^{n+4},
\label{eq:state}
\end{equation}
where $\mbox{pull}_i^{(t)} \in [-1, 1]$ are defined below, $\tau_0$ is a fixed time constant, $1-e^{-t/\tau_0}$ represents the saturating time feature, the trailing constant supplies the intercept of the linear head and
together with the saturating time feature it is independent of $\Theta$. The proposed structured PQC consists of two sub-circuits. In the first sub-circuit, for each turn $t,$ the output quantum state is given by $$\ket{\psi(t)} = U_r(t)\cdots U_r(1)\ket{\psi(0)}$$ with per-turn block $U_r$ and the initial state $\ket{\psi(0)}$. The readout
block is applied to a copy of this state,
\begin{equation}
\ket{\chi(t)}=U_e(\varphi)\,\bigl(R_Y(\mu)\ \text{on}\ q_n\bigr)\ket{\psi(t)},
\end{equation}
and the risk features are read from $\ket{\chi(t)}$ rather than
$\ket{\psi(t)}$.
The sub-circuits are depicted in Figure \ref{fig:circuit-full}. 


An NLP front end (specified, not built here)
produces a per-turn cue vector
$c_t=(c_{t,0},\dots,c_{t,n-1},c_{t,\f}) \in [0,1]^{n}\times [-1,1]$. 
The trainable parameter vector is
\begin{equation}
\Theta = \bigl(\,\underbrace{w_0,\dots,w_{n-1}}_{\text{cue gains}},\;
\underbrace{b_0,\dots,b_{n-1}}_{\text{cue thresholds}},\;
\rho_z,\rho_y,\;
\underbrace{\varphi_0,\dots,\varphi_{n}}_{\text{entangler}},\;
\mu \,\bigr) \in \mathbb{R}^{3n+4}.
\label{eq:params}
\end{equation}
For instance, for $n=4$ biases, there are $16$ circuit parameters i.e. $|\Theta| = 16$.

At each turn $t,$ for each bias qubit, we apply a rotation gate $R_Y(\Delta_i^t)$ to
$q_i$ with
\begin{equation}
\Delta_i^t=\bigl(\pi-\theta_i^{t-1}\bigr)\,g_i^t,
\qquad
 g_i^t=\sigma\!\left(w_i\,c_{t,i}+b_i\right),
\qquad
\theta_i^t=\theta_i^{t-1}+\Delta_i^t ,
\label{eq:gate}
\end{equation}
where $\sigma(x)=1/(1+e^{-x})$ is the logistic function and $\theta_i^0=0$. Note that, since
$g_i^t\in(0,1)$, the cumulative angle satisfies $\theta_i^t<\pi$ for every
$t$ by construction, and hence the monotonicity of
$\text{pull}_i^t=-\langle Z_i\rangle_t=-\cos\theta_i^t$ on $[0,\pi]$ is enforced.
Unrolling \eqref{eq:gate} gives a closed form that makes the semantics
explicit,
\begin{equation}
\theta_i^t=\pi\Bigl(1-\prod_{s\le t}\bigl(1-g_i^s\bigr)\Bigr),
\label{eq:softor}
\end{equation}
i.e.\ the bias qubit's angle is a soft OR, a differentiable relaxation of
the Boolean OR over turns, saturating as the learned probability that at least
one turn so far carried cue $i$. Both the gain $w_i$ and the threshold $b_i$
are trainable parameters. The parameter $b_i$ plays the role of the hand-set cue threshold of
a fixed encoder and $w_i$ that of its gain, but one per channel and fitted to
data rather than shared and guessed. 
The monotonic $\text{pull}_i$ is given by $\text{pull}_i=-\langle Z_i\rangle=-\cos(\theta_i)\in[-1,1]$ for $i=0,\hdots, n-1$ (see Eqn. \ref{eqn:expZ}).

Let us denote $[x]_+=\max(x,0)$, $\zeta_z=\mathrm{softplus}(\rho_z)$, and
$\zeta_y=\mathrm{softplus}(\rho_y)$ for some trainable parameters $\rho_z$ and $\rho_y$. For the frame qubit, at each turn $t$ apply 
\begin{equation}
R_Y(v_t)\,R_Z(u_t),
\qquad
u_t=\zeta_z\,[\,c_{t,\f}\,]_+ ,
\qquad
v_t=\zeta_y\,[-c_{t,\f}\,]_+ ,
\label{eq:frame}
\end{equation} where $\mathrm{softplus}(x)=\log(1+e^x)$ with $\frac{d}{dx}\mathrm{softplus}(x)=\sigma(x).$
Thus a loss-framing term $c_{t,\f}>0$ drives the $Z$-axis, whereas a recovery framing turn $c_{t,\f}<0$ the $Y$-axis with trainable gains. Further, note that the application of $R_Y(\mu)$ on $q_\f$ in the Read-out block realizes the observable \begin{equation}
M(\mu)=R_Y^{\dagger}(\mu)\,Z\,R_Y(\mu)=\cos\mu\;Z-\sin\mu\;X ,
\qquad
\text{pull}_\f^t=-\langle M(\mu)\rangle_t .
\label{eq:mu}
\end{equation}
Finally, observe from Proposition~\ref{prop:orderblind} that at
$\mu=\pm\pi/2$ the frame channel is provably order-blind.

Now we introduce trainable entangler gates that interact with the narrative qubit $q_\n.$
After the per-turn block, each of $q_0,\dots,q_{n-1},q_\f$ acts as the control of a
controlled-$R_Y$ into $q_{\n}$, defined by 
\begin{equation}
U_e(\varphi)=\prod_{j=0}^{n}
\mathrm{C}R_Y^{\,q_j\to q_{n+1}}(\varphi_j),
\qquad
\text{pull}^t_{\mathrm{nar}}=-\langle Z_{n+1}\rangle_t,
\label{eq:ent}
\end{equation}
where $j=0,\dots,n-1$ indexes the bias qubits and $j=n$ the frame qubit, so
there are $n+1$ controls in all.
At $\varphi_j=\pi$ each gate is $\mathrm{CNOT}$. This entangler facilitates the strength of each bias's contribution to narrative coherence, which is a learned quantity.

The qubit and gate complexities scale linearly with the number of cognitive
biases. For $n$ biases the register holds $n+2$ qubits, the per-turn block
$U_r(t)$ contributes $n+2$ single-qubit rotations and no entangling gates, and
the readout block contributes one further rotation together with $n+1$
two-qubit controlled rotations. The circuit producing $s_t$ therefore comprises
$t(n+2)+1$ rotations and $n+1$ controlled rotations.

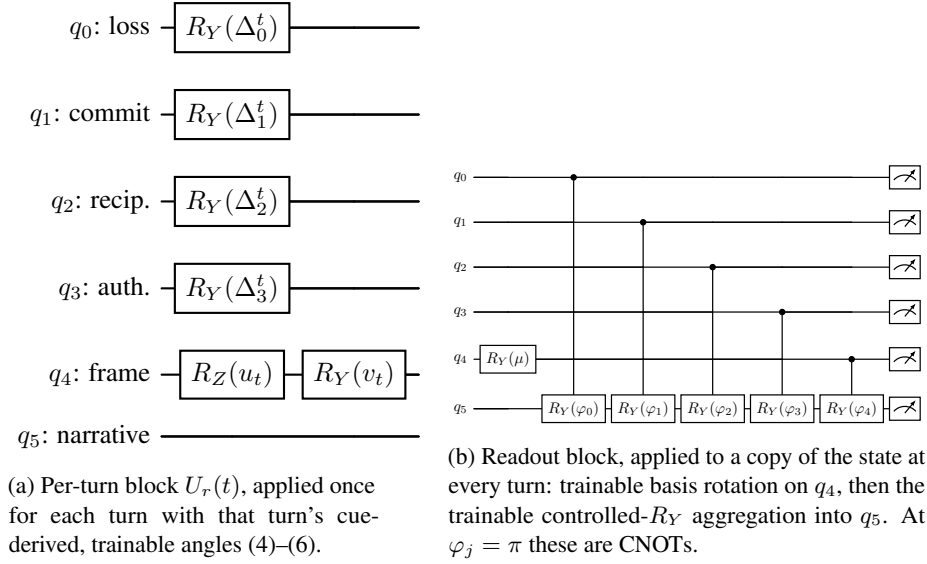
\begin{figure*}[ht]
\centering
\begin{subfigure}[b]{0.40\linewidth}
\centering
\begin{quantikz}[column sep=5pt]
\lstick{$q_0$: loss}      & \gate{R_Y(\Delta_0^t)} & \qw & \qw \\
\lstick{$q_1$: commit}    & \gate{R_Y(\Delta_1^t)} & \qw & \qw \\
\lstick{$q_2$: recip.}    & \gate{R_Y(\Delta_2^t)} & \qw & \qw \\
\lstick{$q_3$: auth.}     & \gate{R_Y(\Delta_3^t)} & \qw & \qw \\
\lstick{$q_4$: frame}     & \gate{R_Z(u_t)} & \gate{R_Y(v_t)} & \qw \\
\lstick{$q_5$: narrative} & \qw & \qw & \qw
\end{quantikz}
\caption{Per-turn block $U_r(t)$, applied once for each turn with that turn's
cue-derived, trainable angles \eqref{eq:gate}--\eqref{eq:frame}.}
\label{fig:circuit-turn}
\end{subfigure}
\hfill
\begin{subfigure}[b]{0.52\linewidth}
\centering
\adjustbox{max width=\linewidth}{%
\begin{quantikz}[column sep=4pt]
\lstick{$q_0$} & \qw & \ctrl{5} & \qw & \qw & \qw & \qw & \meter{} \\
\lstick{$q_1$} & \qw & \qw & \ctrl{4} & \qw & \qw & \qw & \meter{} \\
\lstick{$q_2$} & \qw & \qw & \qw & \ctrl{3} & \qw & \qw & \meter{} \\
\lstick{$q_3$} & \qw & \qw & \qw & \qw & \ctrl{2} & \qw & \meter{} \\
\lstick{$q_4$} & \gate{R_Y(\mu)} & \qw & \qw & \qw & \qw & \ctrl{1} & \meter{} \\
\lstick{$q_5$} & \qw & \gate{R_Y(\varphi_0)} & \gate{R_Y(\varphi_1)}
               & \gate{R_Y(\varphi_2)} & \gate{R_Y(\varphi_3)}
               & \gate{R_Y(\varphi_4)} & \meter{}
\end{quantikz}}
\caption{Readout block, applied to a copy of the state at every turn:
trainable basis rotation on $q_4$, then the trainable controlled-$R_Y$
aggregation into $q_5$. At $\varphi_j=\pi$ these are CNOTs.}
\label{fig:circuit-final}
\end{subfigure}
\caption{Gate-level structure of the parameterized quantum-cognitive
encoder. All angles shown are functions of the trainable parameters
$\Theta$ and of the current turn's cues.}
\label{fig:circuit-full}
\end{figure*}

\section{Joint Quantum-Classical Learning}
\label{sec:rl}

We next describe a procedure for jointly learning the PQC parameters through the RL agent. This coupling yields a hybrid quantum-classical learning framework that falls within the broader paradigm of quantum reinforcement learning. 
At every turn $t,$ the agent observes the learned risk feature vector $s_t$ (Eqn. (\ref{eq:state})) which form the Markov Decision Process (MDP) space for the RL framework. We define the action space as
$\mathcal{A}=\{\textsc{Wait},\textsc{Flag}\}$, where \textsc{Flag} terminates the
episode, \textsc{Wait} advances deterministically, so
$\Pr(s_{t+1}\mid s_t,\textsc{Wait})=1$ and the only stochasticity is in which
conversation is drawn. Rewards are defined as $r(s_t,\textsc{Wait})=0$ and
$$r(s_t,\textsc{Flag})=
\begin{cases}
-P_{\text{fa}} & \text{label}=0,\\[2pt]
B\,\beta^{\max(0,\,t-t^\ast)} & \text{label}=1,
\end{cases}$$
with a terminal penalty $P_{\text{miss}}$ for reaching the end of a fraud
conversation without flagging. Here $t^\ast$ is the labelled critical loss
turn, $\beta<1$ discounts late-but-correct detections, $P_{\text{fa}}$ is the
false-alarm penalty and $B$ the detection bonus, and labels $0$ and $1$ denote
benign and fraud conversations respectively. We use $P_{\text{fa}}=20$,
$P_{\text{miss}}=10$, $B=10$, $\beta=0.8$ and $\gamma=0.98$, with a $20/80$
fraud/benign base rate. We note that the base rate is not incidental to this
choice: at a $50/50$ rate with $P_{\text{fa}}=5$ the expected reward of blind
flagging is $0.5(+10)+0.5(-5)=+2.5>0$, so ``always flag at $t=0$'' is the
genuine optimum of the reward function, and Q-learning correctly finds it.

In our proposal, the RL agent approximates the action-value function linearly instead of a neural network based action-value function. Thus, the action-value function is defined as  
$Q_{\Theta,W}(s_t,a)=w_a\cdot s_t(\Theta)$
which  stays linear in the encoder's output. Thus our model can be interpreted as a PQC feature map
composed with a two-row linear head, and all nonlinearity and all representational
capacity beyond that head live in the circuit. Given the semi-gradient
temporal-difference (TD) error
\begin{equation}
\delta_t = r_t + \gamma\max_{a'}Q(s_{t+1},a') - Q(s_t,a_t),
\end{equation}
with the bootstrap target treated as constant, the classical head follows the
usual update $w_{a_t}\leftarrow w_{a_t}+\alpha_w\,\delta_t\,s_t$, and the same
error drives the circuit parameters,
\begin{equation}
\Theta \;\leftarrow\; \Theta
\;+\;\alpha_\Theta\,\delta_t\,\nabla_\Theta Q(s_t,a_t),
\qquad
\nabla_\Theta Q(s_t,a_t)=\sum_{k=1}^{n+2}\,[w_{a_t}]_k\;
\frac{\partial\,[s_t]_k}{\partial\Theta}.
\label{eq:jointupdate}
\end{equation} The details of the learning through gradients mechanism is described in the Supplementary Section \ref{Sec:gradient}. Algorithms~\ref{alg:encode} and \ref{alg:train} give the full procedure.

\begin{algorithm}[t]
\caption{PQC encoding of a conversation ($n$ biases, $n+2$ qubits)}
\label{alg:encode}
\begin{algorithmic}[1]
\Procedure{Encode}{$\{c_t\}_{t=1}^{T}$, $\Theta$}
  \State $\ket{\psi}\gets \ket{0}^{\otimes n}\otimes\ket{+}\otimes\ket{0}$;\quad
         $\theta_i\gets 0$ for $i=0,\dots,n-1$;\quad $\Phi\gets[\;]$
  \For{$t=1$ \textbf{to} $T$}
    \For{$i=0$ \textbf{to} $n-1$}\Comment{bias qubits}
      \State $g\gets\sigma(w_i c_{t,i}+b_i)$;\quad
             $\Delta\gets(\pi-\theta_i)\,g$;\quad $\theta_i\gets\theta_i+\Delta$
      \State $\ket{\psi}\gets \bigl(R_Y(\Delta)\text{ on }q_i\bigr)\ket{\psi}$
    \EndFor
    \State $u\gets\mathrm{softplus}(\rho_z)\,[c_{t,\f}]_+$;\quad
           $v\gets\mathrm{softplus}(\rho_y)\,[-c_{t,\f}]_+$
    \State $\ket{\psi}\gets \bigl(R_Y(v)R_Z(u)\text{ on }q_n\bigr)\ket{\psi}$
           \Comment{frame qubit, non-commuting}
    \State $\ket{\chi}\gets \bigl(R_Y(\mu)\text{ on }q_n\bigr)\ket{\psi}$
           \Comment{readout copy: separate execution}
    \For{$j=0$ \textbf{to} $n$}
      \State $\ket{\chi}\gets \mathrm{C}R_Y^{\,q_j\to q_{n+1}}(\varphi_j)\ket{\chi}$
    \EndFor
    \State $s_t\gets\bigl(-\langle\chi|Z_0|\chi\rangle,\dots,
            -\langle\chi|Z_{n+1}|\chi\rangle,\;1-e^{-t/\tau_0},\;1\bigr)$;
           \quad append $s_t$ to $\Phi$
  \EndFor
  \State \Return $\Phi$
\EndProcedure
\end{algorithmic}
\end{algorithm}

\begin{algorithm}[t]
\caption{Joint training of the PQC encoder and the RL detector}
\label{alg:train}
\begin{algorithmic}[1]
\Require labelled conversations $(\{c_t\},y,t^\ast)$; rates
  $\alpha_w,\alpha_\Theta$; discount $\gamma$; schedule $\epsilon$;
  constants $P_{\text{fa}},P_{\text{miss}},B,\beta$; anchor $\lambda$;
  warm-up $E_0$
\Ensure trained $\Theta$ and $W=(w_{\textsc{Wait}},w_{\textsc{Flag}})$
\State $W\sim\mathcal{N}(0,0.01^2)$;\quad $\Theta\gets\Theta_0$;\quad
       $\nu\gets 0$
\For{epoch $e=1,\dots,E$}
  \For{each conversation, shuffled}
    \State $\Phi\gets$\Call{Encode}{$\{c_t\},\Theta$};\quad
           $J\gets\partial\Phi/\partial\Theta$ by \eqref{eq:pshift} if $e>E_0$;
           \quad $G\gets 0$, $m\gets 0$
    \For{$t=1$ \textbf{to} $T$}
      \State $a_t\gets\epsilon$-greedy w.r.t.\ $Q(s_t,a)=w_a\cdot s_t$
      \State compute $r_t$ and the target; $\delta_t\gets$ TD error
      \State $\bar w\gets w_{a_t}$
             \Comment{retain the pre-update head weights}
      \State $w_{a_t}\gets w_{a_t}+\alpha_w\,\delta_t\,s_t$
      \If{$e>E_0$}
        \State $G\gets G+\delta_t\,\bar w_{1:n+2}^{\top}J_t$;\quad
               $m\gets m+1$
      \EndIf
      \State \textbf{if} $a_t=\textsc{Flag}$ \textbf{then break}
    \EndFor
    \If{$e>E_0$ \textbf{and} $m>0$}
      \State $G\gets G/m-\lambda(\Theta-\Theta_0)$;\quad
             $G\gets G/\max(1,\|G\|)$
      \State $\nu\gets 0.9\nu+0.1\,G^{\odot 2}$;\quad
             $\Theta\gets\Theta+\alpha_\Theta\,G\oslash(\sqrt{\nu}+10^{-6})$
    \EndIf
  \EndFor
\EndFor
\State \Return $\Theta, W$
\end{algorithmic}
\end{algorithm}

\section{Numerical Simulations with Synthetic Data}
\label{sec:testing}


Since we do not have access to real affinity/romance-fraud conversation
transcripts, as it is not reported in the literature to the best of our knowledge, we generate synthetic, explicitly
labelled conversations, described precisely below so the process is fully
reproducible. In our generative process, each conversation has $T$ turns
($T=14$ for the base classes). We divide $T$ into 6 equal-length stages
using boundaries $\text{stage\_bounds} = \lfloor\text{linspace}(0,T,7)\rfloor$
(so, e.g., $T=14$ gives stage boundaries at turns
$0,2,4,7,9,11,14$), mirroring the six-stage structure documented for
recovery/romance-investment scams in the fraud related cognitive psychology literature
\citep{whitty2013,pratkanis2005weapons}: contact, rapport, reciprocity
hook, authority-backed big ask, loss event, recovery pressure. Each stage
$s\in\{0,\ldots,5\}$ has a fixed base cue profile
$b^{(s)}=(b^{(s)}_{\text{loss}},b^{(s)}_{\text{commit}},
b^{(s)}_{\text{recip}},b^{(s)}_{\text{auth}},b^{(s)}_{\text{frame}})$,
given in Table~\ref{tab:profiles} for the \textsc{Fraud} class. 

For every turn $t$ falling in stage $s$ the cue values are drawn as
\begin{equation}
c_{t,i}=\mathrm{clip}\bigl(b^{(s)}_i+\varepsilon,\ \ell_i,u_i\bigr),
\qquad \varepsilon\sim\mathcal{N}(0,\sigma^2),
\label{eq:noisygen}
\end{equation}
independently for each of the five channels, with
$(\ell_i,u_i)=(0,1)$ for the bias channels and $(-1,1)$ for the frame channel.
The generator produces five classes, two positive and three negative.
\textsc{Fraud} follows Table~1 over $T=14$ turns with $\sigma=0.10$ on all
channels. \textsc{SubtleFraud} uses the same staged schedule over $T=18$ turns
with the bias profile scaled by $0.4$ and the frame profile by $0.5$, giving a
slower, lower-intensity escalation. Among the negatives, \textsc{Benign} uses a
single flat profile $(0.08,0.10,0.08,0.15,0.0)$ with $\sigma=0.05$;
\textsc{LegitUrgent} uses a flat profile $(0.35,0.12,0.05,0.20,0.0)$ with
$\sigma=0.08$, representing a bona fide financial conversation with genuinely
elevated loss and urgency discussion but no reciprocity hook, no authority
escalation and no reframing; and \textsc{PushySales} is a staged hard negative
that mimics the funnel's escalation of reciprocity ($0.60$ at the hook stage),
authority ($0.55$) and commitment ($0.60$) but contains neither a loss event
nor any reframing ($c_{t,f}\equiv 0$), with $\sigma=0.10$ on the bias channels
and $0.08$ on the frame channel. Positives are drawn evenly from
\textsc{Fraud} and \textsc{SubtleFraud}; negatives are $35\%$ \textsc{Benign},
$30\%$ \textsc{LegitUrgent} and $35\%$ \textsc{PushySales}. Consequently
\textsc{PushySales} overlaps substantially with fraud when only the four bias
channels are considered, while retaining a benign label. We use $300$ training
and $300$ test conversations at a $20\%$ fraud base rate and report results
averaged over three random seeds.

The initial values for $\rho_z=\rho_y$ are considered as $0.5413=\text{softplus}^{-1}(1)$, so $\zeta_z=\zeta_y=1,$ unit gain i.e. a maximal framing cue $|c_{t,\f}|=1$ produces a one-rotation  on that turn. The equal values encode no prior that one framing direction matters more than the other.


\begin{table}[t]
\centering
\caption{\textsc{Fraud} class base cue profile $b^{(s)}$ by stage (Eq.~\ref{eq:noisygen}).}
\label{tab:profiles}
\scriptsize
\setlength{\tabcolsep}{3pt}
\begin{tabular}{lrrrrr}
\toprule
Stage & loss & commit & recip. & auth. & frame \\
\midrule
Contact           & 0.05 & 0.05 & 0.05 & 0.10 & 0.0 \\
Rapport            & 0.10 & 0.25 & 0.10 & 0.20 & 0.0 \\
Reciprocity hook    & 0.10 & 0.30 & 0.65 & 0.30 & 0.0 \\
Authority + big ask & 0.20 & 0.45 & 0.30 & 0.60 & 0.0 \\
Loss event          & 0.70 & 0.55 & 0.15 & 0.40 & $-0.8$ \\
Recovery pressure   & 0.65 & 0.75 & 0.20 & 0.55 & $+0.8$ \\
\bottomrule
\end{tabular}
\end{table}

\subsection{Simulation results}
\label{sec:results}

First, we introduce te following standard notations used in the sequel. Denoting TP, FN, FP and TN as true positive (label=1, flagged at some run), false negative (label=1, never flagged), false positive (label=0, flagged at some run), true negative (label=0, never flagged), respectively, the False Positive Rule (FPR) is defined as FP/(FP+TN), F1 (score) is defined as F1=(2 $\cdot$ Precision $\cdot$ Recall)/ (Precision $+$ Recall), where Precision=TP/(TP+FP), Recall=TP/(TP+FN). 

In Table~\ref{tab:pqc-ablation} we report that joint training provides only a
marginal benefit when the encoder is already well calibrated. From the
cognitive initialization the trained and frozen encoders achieve comparable
performance ($F_1 = 0.934\pm0.019$ and $0.919\pm0.051$ respectively), a
difference of roughly one standard deviation that we do not read as decisive.
The difference is substantial under the mis-specified initialization, which
represents transfer to a domain with poorly calibrated parameters and rotation
angles. There the frozen encoder, i.e.\ RL-only learning, fails completely
($F_1 = 0.000$), missing all $62$ fraud conversations, whereas end-to-end
training from the same initialization achieves the best performance in the
table, $F_1 = 0.971\pm0.041$ with no false positives. Figure~\ref{fig:learning-curve} shows that the
two runs diverge only after the circuit is unfrozen. The same reversal produces no change at all in a frame channel that is prepared
in $\ket{0}$ and measured in the $Z$ basis, consistent with
Proposition~\ref{prop:orderblind}.


Further, we note from Table \ref{tab:pqc-ablation} that removing the frame qubit causes detection to collapse ($F_1=0.000$). With only the bias channels, and against
\textsc{PushySales} negatives that mimic the funnel's bias escalation, the
optimal policy under $P_{\text{fa}}=20$ at a $20\%$ base rate is to flag
nothing. Whereas, removing only the narrative qubit is comparatively benign. Figure~\ref{fig:pqc-traj} shows the reason after training, the
frame channel is the one that separates fraud from pushy sales, rising from
$-1$ to $-0.08$ after the loss event while remaining pinned at $-1$ for both
negative classes, whereas the loss and authority channels saturate on all
three. Detection timing is also sequence-dependent at the decision level (Fig.~~\ref{fig:pqc-order}).



\begin{table}[t]
\centering
\caption{Detection quality on the hard synthetic benchmark (300 test
conversations, 20\% fraud base rate, mean $\pm$ s.d.\ over three seeds).
Lateness is the mean flag turn minus the labelled critical loss turn, over
true positives; it is undefined for configurations that never flag.}
\label{tab:pqc-ablation}
\small
\setlength{\tabcolsep}{6pt}
\begin{tabular}{lcccc}
\toprule
Configuration & $F_1$ & FPR & Lateness & FN \\
\midrule
PQC, joint, cognitive init        & $0.934\pm0.019$ & $0.000$ & $+2.19$ & $7.7$ \\
PQC, frozen, cognitive init       & $0.919\pm0.051$ & $0.000$ & $+1.68$ & $9.0$ \\
Classical cue baseline            & $0.899\pm0.020$ & $0.000$ & $+0.99$ & $11.3$ \\
\midrule
PQC, frozen, mis-specified init   & $0.000\pm0.000$ & $0.000$ & --- & $62.0$ \\
\textbf{PQC, joint, mis-specified init} & $\mathbf{0.971\pm0.041}$ & $0.000$ & $+1.37$ & $3.3$ \\
\midrule
PQC, joint, no frame qubit        & $0.000\pm0.000$ & $0.000$ & --- & $62.0$ \\
PQC, joint, no narrative qubit    & $0.937\pm0.015$ & $0.000$ & $+2.19$ & $7.3$ \\
PQC, joint, bias qubits only      & $0.000\pm0.000$ & $0.000$ & --- & $62.0$ \\
\bottomrule
\end{tabular}
\end{table}

\begin{figure*}[ht]
    \centering
    \begin{subfigure}[t]{0.51\textwidth}
        \centering
        \includegraphics[width=\linewidth]{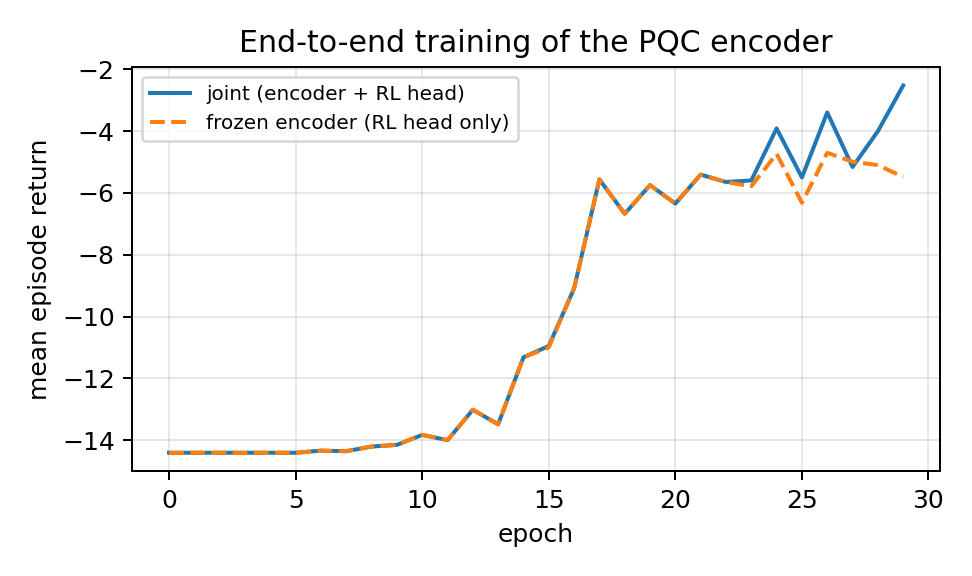}
        \caption{Joint versus frozen training }
        \label{fig:learning-curve}
    \end{subfigure}
    \hfill
    \begin{subfigure}[t]{0.48\textwidth}
        \centering
        \includegraphics[width=\linewidth]{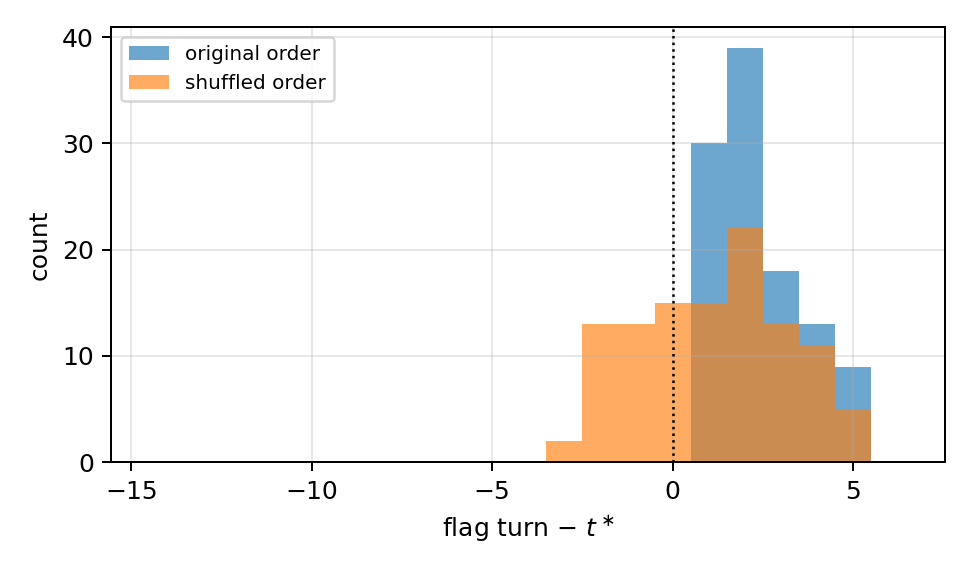}
        \caption{Effect of turn-order shuffling on detection timing}
        \label{fig:pqc-order}
    \end{subfigure}
    \caption{Training and order-sensitivity results. 
    (a) Joint training enables the mis-specified encoder to adapt after
    warm-up, whereas the frozen encoder fails to flag.
    (b) Detection timing under turn-order shuffling with cue
        content held fixed.}
    \label{fig:pqc-training-order}
\end{figure*}

\begin{figure*}[t]
\centering
\includegraphics[width=0.92\linewidth]{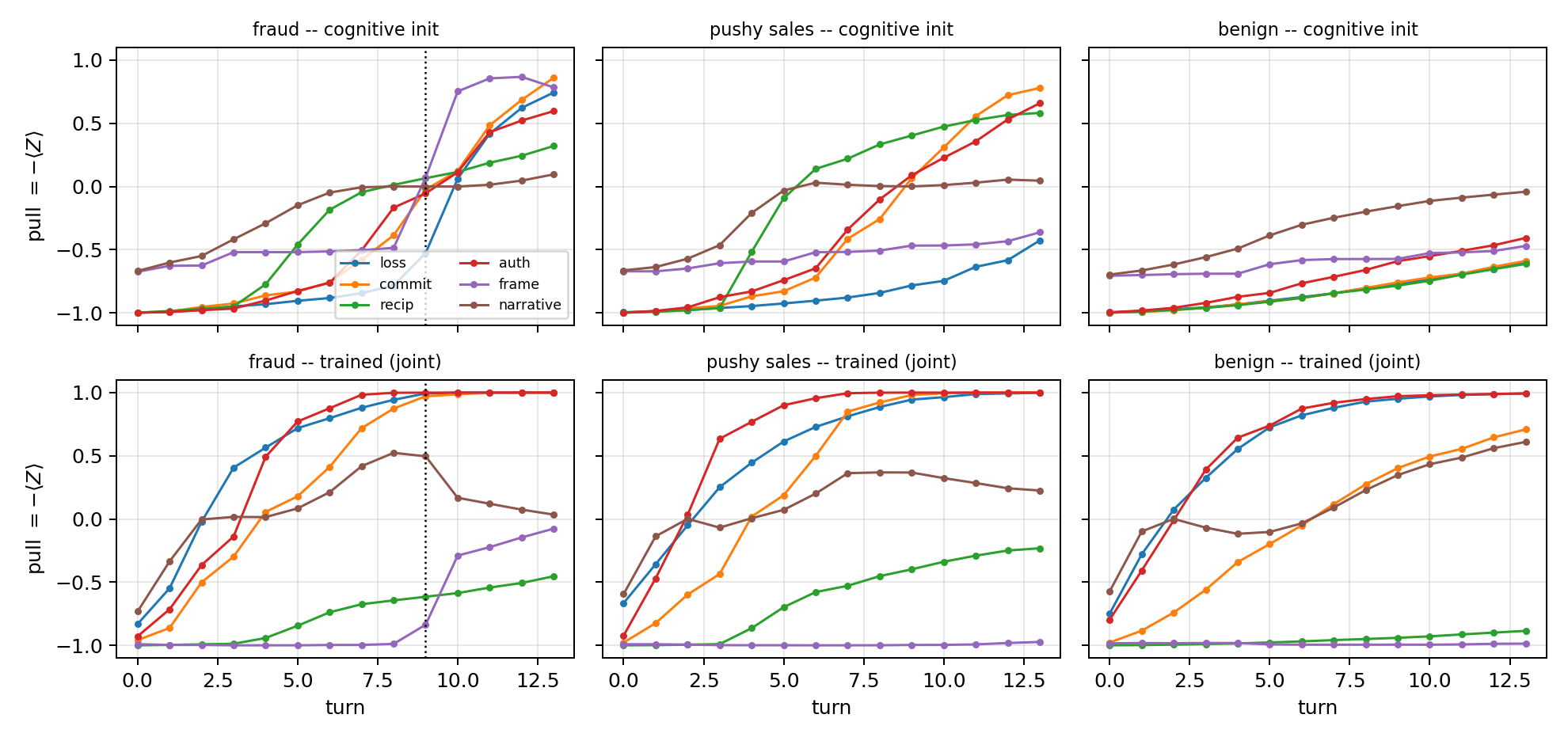}
\caption{Risk-feature trajectories at the cognitive initialization (top) and
after joint training (bottom), for a fraud conversation, a pushy-sales hard
negative, and a benign conversation. Training sharpens the frame channel
into the discriminator that separates fraud from pushy sales, while the loss
and authority channels saturate on all three classes.}
\label{fig:pqc-traj}
\end{figure*}

\section{Limitations and Future Work}
\label{sec:limits}

Several limitations of the present framework motivate future work. 
First, a quantum hardware implementation introduces additional cost and noise. Estimating $\langle Z_i\rangle$ requires finite-shot measurements, and measurement at turn $t$ destroys the quantum state; hence, a $T$-turn trajectory requires $T$ independent circuit executions, with total cost $O(T\times\text{shots})$ per conversation. Training further multiplies this cost by the $2|\mathcal{G}|$ shifted evaluations required by \eqref{eq:pshift}, while shot noise in the TD error $\delta_t$ compounds the intrinsic variance of TD learning. A systematic shot-noise analysis is therefore left for future work. Next, joint optimization remains fragile because the semi-gradient TD update operates on a moving feature map. Warm-up, anchoring, and adaptive gradient scaling are required for stability, with the anchor strength introducing a tradeoff between stability and adaptability, policy-gradient methods or fitted-Q learning with a target network may provide more robust alternatives. Finally, the present evaluation relies entirely on synthetic conversations constructed around the funnel structure encoded by the model. 
In addition, the NLP front end that maps raw text to per-turn cue vectors is currently specified but not independently implemented or validated, and errors in cue extraction would propagate directly to the circuit through the angle map in \eqref{eq:gate}.

\section{Conclusion}

We proposed a hybrid quantum-classical pipeline in which a structured
parameterized quantum circuit encodes manipulation-relevant cognitive biases
such as loss aversion, commitment, reciprocity, authority together with an
order-sensitive frame qubit and a narrative-coherence aggregator, and is
trained jointly with a classical reinforcement-learning agent that learns
when to flag a conversation. On synthetic conversations
with hard negatives, an encoder that is useless when frozen and at a
mis-specified initialization becomes the best-performing configuration when
trained end-to-end, and the frame qubit proves necessary for detection.

\bibliographystyle{plain}
\bibliography{Refference_QML-CRL}

@inproceedings{buetgolfouse2025brl,
  title={Behavioural Reinforcement Learning (Beyond Rationality: RL Under Investor Bias)},
  author={Buet-Golfouse, Francois and Shelley, Osian and Barbulescu, George-Octavian},
  booktitle={Proceedings of the 6th ACM International Conference on AI in Finance},
  pages={344--351},
  year={2025}
}

@incollection{kahneman2013prospect,
  title={Prospect theory: An analysis of decision under risk},
  author={Kahneman, Daniel and Tversky, Amos},
  booktitle={Handbook of the fundamentals of financial decision making: Part I},
  pages={99--127},
  year={2013},
  publisher={World Scientific}
}

@book{tversky1974,
  title={Judgment under uncertainty: Heuristics and biases},
  author={Tversky, Amos and Kahneman, Daniel},
  journal={Science, New Series},
  volume={185},
  number={4157},
  pages={1124--1131},
  publisher={ American Association for the Advancement of Scienc}
}

@article{tversky1983conjunction,
  title={Extensional versus intuitive reasoning: The conjunction fallacy in probability judgment.},
  author={Tversky, Amos and Kahneman, Daniel},
  journal={Psychological review},
  volume={90},
  number={4},
  pages={293},
  year={1983},
  publisher={American Psychological Association}
}

@book{busemeyer2012,
  title={Quantum models of cognition and decision},
  author={Busemeyer, Jerome R and Bruza, Peter D},
  year={2012},
  publisher={Cambridge University Press}
}

@article{busemeyer2011conjunction,
  title={A quantum theoretical explanation for probability judgment errors.},
  author={Busemeyer, Jerome R and Pothos, Emmanuel M and Franco, Riccardo and Trueblood, Jennifer S},
  journal={Psychological review},
  volume={118},
  number={2},
  pages={193},
  year={2011},
  publisher={American Psychological Association}
}

@book{cialdini2001,
  title={Influence: Science and practice},
  author={Cialdini, Robert and Cialdini, Robert B},
  year={2013},
  publisher={Allyn \& Bacon, Boston, MA.}
}

@article{perribrody2011,
  title={Birds of the same feather: the dangers of affinity fraud},
  author={Perri, Frank S and Brody, Richard G},
  journal={Journal of Forensic Studies in Accounting and Business},
  volume={3},
  number={1},
  pages={33--46},
  year={2011}
}

@article{whitty2013,
  title={The online romance scam: A serious cybercrime},
  author={Whitty, Monica T and Buchanan, Tom},
  journal={CyberPsychology, Behavior, and Social Networking},
  volume={15},
  number={3},
  pages={181--183},
  year={2012},
  publisher={SAGE Publications Sage CA: Los Angeles, CA}
}

@article{buchananwhitty2014,
  title={The online dating romance scam: causes and consequences of victimhood},
  author={Buchanan, Tom and Whitty, Monica T},
  journal={Psychology, Crime \& Law},
  volume={20},
  number={3},
  pages={261--283},
  year={2014},
  publisher={Taylor \& Francis}
}

@article{leafischerevans2009,
  title={The psychology of scams: Provoking and committing errors of judgement},
  author={Lea, Stephen EG and Fischer, Peter and Evans, Kath M},
  year={2009},
  publisher={University of Exeter}
}

@article{pratkanis2005weapons,
  title={Weapons of Fraud: A Source Book for Fraud Fighters.},
  author={Anthony Pratkanis and Doug Shadel},
  year={2005},
  publisher={AARP Washington.}
  
}

@misc{securities2013affinity,
  title={Affinity fraud: how to avoid investment scams that target groups},
  author={Securities, US and Exchange Commission and others},
  year={2013}
}

@article{coluccia2020online,
  title={Online romance scams: Relational dynamics and psychological characteristics of the victims and scammers. A scoping review},
  author={Coluccia, Anna and Pozza, Andrea and Ferretti, Fabio and Carabellese, Fulvio and Masti, Alessandra and Gualtieri, Giacomo},
  journal={Clinical practice and epidemiology in mental health: CP \& EMH},
  volume={16},
  pages={24},
  year={2020}
}

@article{dearden2025mechanisms,
  title={Mechanisms of affinity fraud victimization},
  author={Dearden, Thomas and Jasperson, Jill O and Miller, Ronald Mellado},
  journal={Journal of Financial Crime},
  volume={32},
  number={1},
  pages={64--76},
  year={2025},
  publisher={Emerald Publishing Limited}
}

@article{busemeyer2015quantum,
  title={What is quantum cognition, and how is it applied to psychology?},
  author={Busemeyer, Jerome R and Wang, Zheng},
  journal={Current Directions in Psychological Science},
  volume={24},
  number={3},
  pages={163--169},
  year={2015},
  publisher={Sage Publications Sage CA: Los Angeles, CA}
}

@article{widdows2023quantum,
  title={Quantum circuit components for cognitive decision-making},
  author={Widdows, Dominic and Rani, Jyoti and Pothos, Emmanuel},
  journal={arXiv preprint arXiv:2302.03012},
  year={2023}
}

@book{nielsen2001quantum,
  title={Quantum computation and quantum information},
  author={Nielsen, Michael A and Chuang, Isaac L},
  volume={2},
  year={2001},
  publisher={Cambridge university press Cambridge}
}

@article{ruggeri2023persistence,
  title={The persistence of cognitive biases in financial decisions across economic groups},
  author={Ruggeri, Kai and Ashcroft-Jones, Sarah and Abate Romero Landini, Giampaolo and Al-Zahli, Narjes and Alexander, Natalia and Andersen, Mathias Houe and Bibilouri, Katherine and Busch, Katharina and Cafarelli, Valentina and Chen, Jennifer and others},
  journal={Scientific Reports},
  volume={13},
  number={1},
  pages={10329},
  year={2023},
  publisher={Nature Publishing Group UK London}
}

@article{de2024financial,
  title={Financial scarcity and cognitive performance: A meta-analysis},
  author={de Almeida, Filipa and Scott, Ian J and Soro, Jer{\^o}nimo C and Fernandes, Daniel and Amaral, Andr{\'e} R and Catarino, Mafalda L and Ar{\^e}de, Andr{\'e} and Ferreira, M{\'a}rio B},
  journal={Journal of Economic Psychology},
  volume={101},
  pages={102702},
  year={2024},
  publisher={Elsevier}
}

@article{ruggeri2024financial,
  title={Financial decision-making, income, cognitive biases: The impact of economic systems and environments on behavior in six countries},
  author={Ruggeri, Kai and Abate Romero Landini, Giampaolo and Busch, Katharina and Cafarelli, Valentina and Doubravov{\'a}, Barbora and Gurol, Deniz and Miralem, Melika and Nilsson, Fredrik and Ashcroft-Jones, Sarah and Stock, Friederike},
  year={2024},
  publisher={OSF}
}

@article{hammad2026cognitive,
  title={Cognitive Biases and Portfolio Decisions: A Neurofinance Perspective on Investor Irrationality},
  author={Hammad, VP Syed Abdullahil and Marimuthu, M},
  journal={Scriptora International Journal of Research and Innovation (SIJRI)},
  pages={14--23},
  year={2026}
}

@article{tkacz2025exploration,
  title={Exploration of Behavioral Patterns and Cognitive Biases among Stock Market Investors},
  author={Tkacz, Maciej and Olszak, Celina and Zurada, Jozef and Weichbroth, Pawe{\l}},
  year={2025}
}

@article{gabhane2023behavioral,
  title={Behavioral finance: exploring the influence of cognitive biases on investment decisions},
  author={Gabhane, Dinesh and Sharma, A and Mukherjee, Rupam},
  journal={Boletin de Literatura Oral-The Literary Journal},
  volume={10},
  number={1},
  pages={3133--3141},
  year={2023}
}

@article{cardie2000cognitive,
  title={A cognitive bias approach to feature selection and weighting for case-based learners},
  author={Cardie, Claire},
  journal={Machine Learning},
  volume={41},
  number={1},
  pages={85--116},
  year={2000},
  publisher={Springer}
}

@inproceedings{kaleem2017cognitive,
  title={A cognitive approach for attribute selection in internet dataset},
  author={Kaleem, Danish and Ferens, Ken},
  booktitle={2017 IEEE 16th International Conference on Cognitive Informatics \& Cognitive Computing (ICCI* CC)},
  pages={319--328},
  year={2017},
  organization={IEEE}
}

@article{abbas2021power,
  title={The power of quantum neural networks},
  author={Abbas, Amira and Sutter, David and Zoufal, Christa and Lucchi, Aur{\'e}lien and Figalli, Alessio and Woerner, Stefan},
  journal={Nature computational science},
  volume={1},
  number={6},
  pages={403--409},
  year={2021},
  publisher={Nature Publishing Group US New York}
}

@article{van2018deep,
  title={Deep reinforcement learning and the deadly triad},
  author={Van Hasselt, Hado and Doron, Yotam and Strub, Florian and Hessel, Matteo and Sonnerat, Nicolas and Modayil, Joseph},
  journal={arXiv preprint arXiv:1812.02648},
  year={2018}
}

@article{meyer2022survey,
  title={A survey on quantum reinforcement learning},
  author={Meyer, Nico and Ufrecht, Christian and Periyasamy, Maniraman and Scherer, Daniel D and Plinge, Axel and Mutschler, Christopher},
  journal={arXiv preprint arXiv:2211.03464},
  year={2022}
}

@article{chen2026quantum,
  title={Quantum reinforcement learning: Concepts and applications},
  author={Chen, Samuel Yen-Chi},
  journal={Quantum Computational AI},
  pages={3--23},
  year={2026},
  publisher={Elsevier}
}

@article{franco2009conjunction,
  title={The conjunction fallacy and interference effects},
  author={Franco, Riccardo},
  journal={Journal of Mathematical Psychology},
  volume={53},
  number={5},
  pages={415--422},
  year={2009},
  publisher={Elsevier}
}

@article{wang2006order,
  title={The order effect in human abductive reasoning: An empirical and computational study},
  author={Wang, Hongbin and Johnson, Todd R and Zhang, Jiajie},
  journal={Journal of Experimental \& Theoretical Artificial Intelligence},
  volume={18},
  number={2},
  pages={215--247},
  year={2006},
  publisher={Taylor \& Francis}
}

@article{trueblood2011quantum,
  title={A quantum probability account of order effects in inference},
  author={Trueblood, Jennifer S and Busemeyer, Jerome R},
  journal={Cognitive science},
  volume={35},
  number={8},
  pages={1518--1552},
  year={2011},
  publisher={Wiley Online Library}
}

@article{khrennikova2014order,
  title={Order effect in a study on US voters’ preferences: quantum framework representation of the observables},
  author={Khrennikova, Polina},
  journal={Physica Scripta},
  volume={163},
  number={1},
  pages={014010},
  year={2014},
  publisher={IOP Publishing}
}

\newpage 
\section{Supplementary}

\subsection{Mathematical Foundations}


Recall that an $n$-qubit state is $\ket{\psi}\in\mathbb{C}^{2^n}$. The Pauli
matrices of order $2$ are
$$X=\bmatrix{0&1\\1&0},\qquad
Y=\bmatrix{0&-i\\i&0},\qquad
Z=\bmatrix{1&0\\0&-1},$$
along with the identity $I_2=\begin{pmatrix}1&0\\0&1\end{pmatrix}$. Then the single-qubit
rotations that we consider in this paper are:
$R_Y(\theta)=e^{-i\theta Y/2}=\cos(\theta/2)I-i\sin(\theta/2)Y$, and
similarly for $R_Z$. The observables we consider as the Pauli $Z$ matrix and $\langle Z\rangle=\langle\psi|Z|\psi\rangle\in\mathbb{R}$ denotes the expected value of the measurement outcomes, see \cite{nielsen2001quantum} for basics of quantum computation. 

Now we have the following propositions, whose proofs are computational and easy to verify. The Lie bracket for two operators $A,B$ is defined as $[A, \,B]=AB-BA,$ which is the zero matrix only when they commute. Proposition \ref{prop:noncommute} describes the frame-order dependence, whereas Proposition \ref{prop:frechet} shows that no joint measurement can produce the conjunction fallacy.

\begin{proposition}
\label{prop:noncommute}
$[R_Y(\theta),R_Z(\phi)] = -2i\sin(\theta/2)\sin(\phi/2)\,X$, which vanishes
only for trivial rotation angles; for any other $\theta,\phi$,
$R_Y(\theta)R_Z(\phi)\neq R_Z(\phi)R_Y(\theta)$.
\end{proposition}

\begin{proposition}
\label{prop:frechet}
For any two qubits $q_A,q_B$ in any joint state,
$P(A\wedge B)\le\min(P(A),P(B))$, since Born-rule outcome probabilities
$p_{00},p_{01},p_{10},p_{11}\ge 0$ summing to $1$ give
$P(A\wedge B)=p_{11}\le p_{10}+p_{11}=P(A)$, and symmetrically for $B$,
regardless of entanglement.
\end{proposition}

Proposition~\ref{prop:frechet} shows that the conjunction fallacy cannot be reproduced by a
\textit{single joint measurement} of two qubits, however they are entangled. Indeed, the
Fr\'echet bound follows from non-negativity of the Born-rule outcome
probabilities alone. This does not contradict the quantum-cognitive account of
the fallacy, which obtains it from \textit{sequential incompatible} measurements
rather than a joint one~\cite{busemeyer2011conjunction}. It does, however,
determine our design: the narrative qubit is used not to violate the bound but
as an order-insensitive coherence aggregate, while order sensitivity is carried
entirely by the frame qubit through Proposition~\ref{prop:noncommute}.

Next, we show that the expected value $\langle Z_i\rangle$ for any bias qubit $i$ after applying a rotation gate $R_Y(\theta)$ is given by $\cos\theta$ for the initial state $\ket{0}.$ Note that $\ket{\psi}=R_Y(\theta)\ket{0}=\cos\frac{\theta}{2} \ket{0} +\sin\frac{\theta}{2}\ket{1}.$ Then \begin{equation}\label{eqn:expZ}
\langle Z\rangle=\langle\psi|Z|\psi\rangle=\cos^2\frac{\theta}{2}-\sin^2\frac{\theta}{2}=\cos\theta.\end{equation} Then applying a sequence of $R_Y$ gates, only amplify the rotation angle, and hence it remains cosine of the total rotation angle. Then note that $-\langle Z_i\rangle=-\cos\theta_i$ is a strictly monotonic function on $[0, \, \pi].$

Next we have the following Proposition, which establishes the order-blind of the measurement bases.

\begin{proposition}
\label{prop:orderblind}
Let the frame qubit be prepared in $\ket{+}$ and let one loss-framing event
$R_Z(u)$ and one recovery-framing event $R_Y(v)$ occur, in either order.
Write $\langle\cdot\rangle_{ZY}$ for the state $R_Y(v)R_Z(u)\ket{+}$ and
$\langle\cdot\rangle_{YZ}$ for $R_Z(u)R_Y(v)\ket{+}$. Then
\begin{equation}
\langle X\rangle_{ZY}=\langle X\rangle_{YZ}=\cos u\cos v ,
\qquad
\langle Z\rangle_{ZY}-\langle Z\rangle_{YZ}=\sin v\,(1-\cos u),
\end{equation}
and consequently, for the observable $M(\mu)$ of \eqref{eq:mu},
\begin{equation}
\bigl|\langle M(\mu)\rangle_{ZY}-\langle M(\mu)\rangle_{YZ}\bigr|
=\bigl|\cos\mu\bigr|\;\sin v\,(1-\cos u).
\label{eq:orderblind}
\end{equation}
The order effect is therefore  zero at $\mu=\pm\pi/2$ and
maximal at $\mu\in\{0,\pi\}$. Moreover, for a frame qubit prepared in
$\ket{0}$ and measured in the $Z$ basis, any prefix consisting solely of
$R_Z$ rotations leaves $\langle Z\rangle$ unchanged, so such a frame channel
cannot register loss-framing that is not followed by recovery-framing.
\end{proposition}

\begin{proof}
On the Bloch sphere $\ket{+}$ is $(1,0,0)$. $R_Z(u)$ rotates by $u$ about
$z$, $R_Y(v)$ by $v$ about $y$. Applying $R_Z$ then $R_Y$ gives
$(\cos u\cos v,\ \sin u,\ -\cos u\sin v)$; applying $R_Y$ then $R_Z$ gives
$(\cos v\cos u,\ \cos v\sin u,\ -\sin v)$. The $x$ components agree and the
$z$ components differ by $\sin v-\cos u\sin v$. Substituting into
$\langle M(\mu)\rangle=\cos\mu\,\langle Z\rangle-\sin\mu\,\langle X\rangle$
gives \eqref{eq:orderblind}. The final claim is immediate from
$R_Z(u)\ket{0}=e^{-iu/2}\ket{0}$.
\end{proof}

Proposition~\ref{prop:orderblind} is the reason the readout angle $\mu$ is a
trained parameter and is initialized at $-\pi/4$ rather than at the
natural-looking $-\pi/2$ (which measures $X$ exactly, and is precisely the
blind point). 


The following Proposition provides a closed form expression of the narrative qubit's readout.

\begin{proposition}
\label{prop:closedform}
Every gate of the per-turn block acts on a single qubit, so
$\ket{\phi(t)}$ is a product state $\bigotimes_j\ket{\phi_j(t)}$. Let
$p_j=|\langle 1|\phi_j(t)\rangle|^2$ for $j=0,\dots,n$ (for $j=n$, after the
readout rotation $R_Y(\mu)$). Then the gates of $U_e(\varphi)$ in
\eqref{eq:ent} pairwise commute, each leaves its own control's
$\langle Z_j\rangle$ invariant, and
\begin{equation}
\langle Z_{n+1}\rangle
=\mathrm{Re}\prod_{j=0}^{n}\Bigl(1+p_j\bigl(e^{i\varphi_j}-1\bigr)\Bigr).
\label{eq:narrclosed}
\end{equation}
\end{proposition}

\begin{proof}
All the gates of $U_e$ share the target $q_{n+1}$ and rotate it about the
same axis, and their controls are diagonal, so they commute pairwise; a
controlled gate is diagonal on its control, hence leaves $\langle Z_j\rangle$
unchanged. Conditioned on the control bitstring $z\in\{0,1\}^{n+1}$, which
occurs with probability $\prod_j p_j^{z_j}(1-p_j)^{1-z_j}$ because the
controls are in a product state, the target evolves as
$R_Y\bigl(\sum_j z_j\varphi_j\bigr)\ket{0}$ and contributes
$\cos\bigl(\sum_j z_j\varphi_j\bigr)$ to $\langle Z_{n+1}\rangle$. Averaging
and writing $\cos x=\mathrm{Re}\,e^{ix}$ factorises the sum over $z$ into the
product \eqref{eq:narrclosed}.
\end{proof}

\begin{remark}
\label{cor:sim}
By Proposition~\ref{prop:closedform}, the full feature trajectory
$\{s_t\}_{t=1}^{T}$ can be evaluated in $O(nT)$ classical arithmetic
operations, without forming the $2^{n+2}$-dimensional statevector. The role
of the quantum formalism here is representational. It supplies a
non-commutative, interference-based inductive bias matched to documented
judgement phenomena not computational. Quantum hardware implementation of the proposal would invite additional challenges, which is left as future work.
\end{remark}

Next we have the following propositions, which shows that the entangler defined by CNOT gates is stationary.

\begin{proposition}
\label{prop:cnotstat}
With $\langle Z_{n+1}\rangle$ as in \eqref{eq:narrclosed} and
$f_l=1+p_l(e^{i\varphi_l}-1)$,
\[
\frac{\partial \langle Z_{n+1}\rangle}{\partial \varphi_j}
=\mathrm{Re}\Bigl[i\,p_j e^{i\varphi_j}\prod_{l\neq j} f_l\Bigr].
\]
At $\varphi=\pi\mathbf{1}$ every $f_l=1-2p_l$ is real, so the bracket is purely
imaginary and the derivative vanishes for every $j$, every $t$, and every input.
The point $\varphi=\pi\mathbf{1}$ is thus a critical point of the narrative
readout and the second derivative is strictly negative. Consequently, it is a local maximum
of the narrative contrast.
\end{proposition}
 
Observe that the 
Proposition~\ref{prop:cnotstat} explains why the CNOT aggregation of a fixed
encoder could be a sensible design choice for the encoder. Besides, it also establishes our observation in numerical simulations that the entangler
angles do not move when training starts from $\Theta_0$, in fact the the first-order signal
is identically zero there. Accordingly, the evidence that training $\varphi$
helps comes only from the mis-specified runs, which start away from the
stationary point. 
Initializing $\varphi_j$ such as $\pi-0.2$
removes the degeneracy at negligible cost to interpretability.

\subsection{Gradients through the circuit and the joint training}\label{Sec:gradient}

Note that Eqn.~\eqref{eq:jointupdate} is the entire quantum-classical interface. The RL agent's surprise at turn $t$ is projected through its own weights onto
the Jacobian of the Born-rule expectations with respect to the gate
parameters. Besides, nothing about the conversation reaches the agent except the
register's readout, and nothing about the labels reaches the circuit except
$\delta_t$.

The Jacobian is obtained on hardware by the parameter-shift rule composed
with the chain rule through the classical angle map. Writing
$\mathcal{G}$ for the set of parameterised gate occurrences in the circuit
that produces $s_t$, and $\eta_g$ for the angle of occurrence $g$,
\begin{equation}
\frac{\partial \langle O\rangle}{\partial\Theta}
=\sum_{g\in\mathcal{G}}
\underbrace{\tfrac12\Bigl[\langle O\rangle_{\eta_g+\pi/2}
-\langle O\rangle_{\eta_g-\pi/2}\Bigr]}_{\text{parameter shift}}
\;\cdot\;
\frac{\partial \eta_g}{\partial\Theta},
\label{eq:pshift}
\end{equation}
which costs $2|\mathcal{G}|=2\bigl(t(n+2)+(n+2)\bigr)$ circuit evaluations
per turn, for any observable $O$. Note that the second factor is not a plain derivative because the angle map
\eqref{eq:gate} is recurrent since $\Delta_i^t$ depends on $\theta_i^{t-1}$, and
hence on all earlier cues but the recursion admits a closed form representation. The following lemma provides gradients of the saturating accumulation.

\begin{lemma}
\label{lem:accgrad}
With $g_i^s=\sigma(w_i c_{s,i}+b_i)$ and $\theta_i^t$ as in
\eqref{eq:softor}, and writing $P_i^t=\prod_{s\le t}(1-g_i^s)$,
\begin{equation}
\frac{\partial\theta_i^t}{\partial w_i}
=\pi\,P_i^t\sum_{s\le t} g_i^s\,c_{s,i},
\qquad
\frac{\partial\theta_i^t}{\partial b_i}
=\pi\,P_i^t\sum_{s\le t} g_i^s .
\end{equation}
\end{lemma}

\begin{proof}
$\partial\log(1-g_i^s)/\partial w_i=-\,g_i^s(1-g_i^s)c_{s,i}/(1-g_i^s)
=-g_i^s c_{s,i}$, so
$\partial P_i^t/\partial w_i=-P_i^t\sum_{s\le t}g_i^s c_{s,i}$; substitute
into $\theta_i^t=\pi(1-P_i^t)$. The $b_i$ case is identical with
$c_{s,i}$ replaced by $1$.
\end{proof}

If both running sums are accumulated online, the classical factor in
\eqref{eq:pshift} costs $O(n)$ per turn. In the simulations of
Section~\ref{sec:testing} we evaluate the same derivative exactly and
analytically from Proposition~\ref{prop:closedform} and
Lemma~\ref{lem:accgrad}; we verified this implementation against central
finite differences over all $16$ parameters and all $6$ observables at every
turn, with maximum absolute discrepancy $1.0\times10^{-9}$. 


\subsection{Algorithm~\ref{alg:train}}
\label{sec:alg2-walkthrough}
 
In this section we explain the execution of the Algorithm~\ref{alg:train}. First, we include all the notations in Table~\ref{tab:notation}.  

\begin{table}[t]
\centering
\caption{Notation used in Algorithm~\ref{alg:train}.}
\label{tab:notation}
\small
\setlength{\tabcolsep}{5pt}
\begin{tabular}{ll}
\toprule
Symbol & Meaning \\
\midrule
\multicolumn{2}{l}{Episode (one labelled conversation)}\\
$T$ & number of turns \\
$c_t\in[0,1]^{n}\!\times\![-1,1]$ & cue vector at turn $t$ from the NLP front end \\
$y\in\{0,1\}$ & conversation label; $1$ = fraud \\
$t^\ast$ & labelled critical loss turn (first turn of the loss stage) \\
\midrule
\multicolumn{2}{l}{Quantum encoder}\\
$\Theta\in\mathbb{R}^{3n+4}$ & circuit parameters, Eq.~\eqref{eq:params} \\
$\Theta_0$ & cognitive initialisation ($\varphi_j=\pi$, $\mu=-\pi/4$) \\
$\Phi=(s_1,\dots,s_T)$ & encoded feature trajectory for the episode \\
$s_t\in\mathbb{R}^{n+4}$ & state at turn $t$, Eq.~\eqref{eq:state} \\
$J_t=\partial (s_t)_{1:n+2}/\partial\Theta$ & Jacobian of the Born-rule
   features, $(n{+}2)\times(3n{+}4)$ \\
\midrule
\multicolumn{2}{l}{Reinforcement learner}\\
$\mathcal{A}=\{\textsc{Wait},\textsc{Flag}\}$ & action set; \textsc{Flag} is terminal \\
$W=(w_{\textsc{Wait}},w_{\textsc{Flag}})$ & linear head, $w_a\in\mathbb{R}^{n+4}$ \\
$Q(s_t,a)=w_a\!\cdot\! s_t$ & action-value function \\
$a_t$, $r_t$ & action taken and reward received at turn $t$ \\
$\delta_t$ & temporal-difference error, Eq.~\eqref{eq:tderr} \\
$\gamma=0.98$ & discount factor \\
$\epsilon$ & exploration probability, annealed $1\!\to\!0.05$ \\
$P_{\text{fa}},P_{\text{miss}},B,\beta$ & false-alarm penalty, miss penalty,
   detection bonus, lateness discount \\
\midrule
\multicolumn{2}{l}{Optimization}\\
$E$, $E_0$ & total epochs ($30$) and warm-up epochs ($\lfloor 0.34E\rfloor=10$) \\
$\alpha_w=0.02$, $\alpha_\Theta=0.003$ & learning rates for head and circuit \\
$G\in\mathbb{R}^{3n+4}$, $m$ & accumulated circuit gradient and its step count \\
$\lambda=0.02$ & anchor strength toward $\Theta_0$ \\
$\nu$ & RMSProp running second moment \\
$\odot,\ \oslash$ & elementwise square, elementwise division \\
\bottomrule
\end{tabular}
\end{table}

The head $W$ is initialized near zero and the circuit at the cognitive point $\Theta_0$, so the very first policy is uninformed but the very first feature map is a hand-designed encoder. Note that $\tau_0>0$ is a fixed time constant, in units of turns, that sets the
scale over which the saturating time feature $1-e^{-t/\tau_0}$ distinguishes
early from late turns. We use $\tau_0=5$, so the feature reads $0.63$ at
$t=\tau_0$ and $0.93$ at the final turn of a $T=14$ conversation. Each epoch sweeps the labelled conversations in a fresh random order. For one conversation the encoder is invoked once (per visit), before the agent acts, and it returns the entire trajectory $\Phi=(s_1,\dots,s_T)$ together with (when the circuit is unfrozen) the stack of per-turn Jacobians $J_1,\dots,J_T$, $J_t=\partial (s_t)_{1:n+2}/\partial\Theta\in\mathbb{R}^{(n+2)\times(3n+4)}$,
the Jacobian of the Born-rule features at turn $t$ with respect to the circuit
parameters; row $k$ is $\partial\bigl(-\langle Z_k\rangle_t\bigr)/\partial\Theta$
and the time and constant components of $s_t$ are omitted as they do not
depend on $\Theta$ (see Eq.~\eqref{eq:pshift}). The agent then replays that pre-computed trajectory turn by turn, choosing $a_t$ at each row.


Note that, during the $E_0$ warm-up epochs, the Jacobians are not formed at all. Indeed,  $\Phi$ costs one forward sweep,
whereas each $J_t$ costs $2|\mathcal{G}|$ shifted circuit evaluations. Thus, where the saturating
recursion $\theta_i^t=\theta_i^{t-1}+\Delta_i^t$ passes through every prefix
on its way to $t=T$ and so yields all $T$ states in a single pass. On
hardware, measurement collapses the state, so the trajectory genuinely
requires $T$ separate circuit executions; what
encoding once per episode avoids there is a further factor of $T$ from
re-encoding inside the agent loop.


 
At turn $t$ the agent evaluates $Q(s_t,a)=w_a\!\cdot\! s_t$ for both actions and
selects $a_t$ $\epsilon$-greedily, meaning that with probability $1-\epsilon$
it takes the currently best-estimated action $\arg\max_{a}Q(s_t,a)$ and with
probability $\epsilon$ it ignores its own estimate and samples uniformly from
$\mathcal{A}=\{\textsc{Wait},\textsc{Flag}\}.$ Thus, 
\begin{equation}
a_t=\begin{cases}
\arg\max_{a\in\mathcal{A}}Q(s_t,a) & \text{with probability } 1-\epsilon,\\[2pt]
\text{Uniform}(\mathcal{A}) & \text{with probability } \epsilon .
\end{cases}
\label{eq:epsgreedy}
\end{equation}
At evaluation we set $\epsilon=0$, so the reported detector is the deterministic
greedy policy $a_t^\ast=\arg\max_a Q(s_t,a)$.


The reward and bootstrap target follow
the three cases of the stopping problem: if $a_t=\textsc{Flag}$ the episode
terminates and $r_t$ is $-P_{\text{fa}}$ on a benign conversation or
$B\beta^{\max(0,t-t^\ast)}$ on a fraudulent one, with no bootstrap term; if
the agent reaches the last turn still waiting, it receives $-P_{\text{miss}}$
when $y=1$ and $0$ otherwise, again with no bootstrap; otherwise $r_t=0$ and
the target is $\gamma\max_{a'}Q(s_{t+1},a')$. Thus the temporal-difference error is given by
\begin{equation}
\delta_t=\bigl[\text{target}\bigr]-Q(s_t,a_t)
\label{eq:tderr}
\end{equation}
which measures how wrong the current estimate was, and it
drives both halves of the model. The classical head takes its usual
semi-gradient step $w_{a_t}\!\leftarrow\! w_{a_t}+\alpha_w\delta_t s_t$
immediately. The circuit's contribution for this turn,
$\delta_t\,(w_{a_t})_{1:n+2}^{\top}J_t$, is not applied immediately but
accumulated into $G$: this is the chain rule of
Eq.~\eqref{eq:jointupdate}, which captures ``how much would each gate angle have had
to change to move $Q(s_t,a_t)$ toward its target". Only the first $n+2$
components of $\bar w$ appear because the time feature and the constant do
not depend on $\Theta$. Choosing \textsc{Flag} breaks the turn loop.
 

Note that both updates are gradients of the
same objective $\tfrac12\bigl(\text{target}-Q(s_t,a_t)\bigr)^2$ under the
semi-gradient convention, so both must be evaluated at the same parameters.
Substituting the already-updated head weights instead perturbs the circuit
gradient by
\begin{equation}
\delta_t\bigl(w_{a_t}^{\text{new}}-\bar w\bigr)_{1:n+2}^{\top}J_t
=\alpha_w\,\delta_t^{2}\,(s_t)_{1:n+2}^{\top}J_t ,
\label{eq:headorder}
\end{equation}
which, being proportional to $\delta_t^{2}\ge 0$, is a sign-definite drift
along $s_t^{\top}J_t$ rather than zero-mean noise. Thus it systematically inflates
whichever channels are currently large, and does so most strongly on exactly
the episodes the agent finds most surprising. 

 

After the episode, $G$ is divided by the number $m$ of turns that contributed,
so a conversation flagged at turn $2$ does not receive a smaller circuit
update than one that ran to turn $14$. The anchor term
$-\lambda(\Theta-\Theta_0)$ is then subtracted, a quadratic pull back toward
the cognitive initialization, $G$ is clipped to unit norm, and the step is
scaled per coordinate by an RMSProp second moment $\nu$, so that parameters
of very different sensitivity, a cue gain $w_i$ and an entangler angle
$\varphi_j$ move the features by quite different amounts, take comparable
steps. Two guards make this stable. The circuit is frozen for the first
$E_0$ epochs, so the feature map does not move while the $Q$-values are still
random, and the anchor bounds how far the learned encoder can drift from the
documented baseline, which is a governance property as much as an
optimization one.
 

Finally, note that nothing about the conversation reaches the agent except the register's
readout $s_t(\Theta)$, and nothing about the labels reaches the circuit except
the scalar $\delta_t$. The encoder is therefore not trained to classify
turns. In fact, it is trained to make the stopping decision easier, under the same
cost asymmetry the deployed detector faces. This distinguishes
Algorithm~\ref{alg:train} from fitting a quantum feature map to labels and
then bolting an agent onto it.
 
It is worth stating explicitly that neither $\Theta$ nor $W$ is trained by
minimizing a loss. Indeed, the per-turn squared error
\begin{equation}
L_t(\Theta,W)=\tfrac12\bigl(y_t-Q_{\Theta,W}(s_t,a_t)\bigr)^2,
\qquad
y_t=\begin{cases}
r_t & \text{terminal step},\\
r_t+\gamma\max_{a'}Q(s_{t+1},a') & \text{otherwise},
\end{cases}
\label{eq:surrogate}
\end{equation}
is a surrogate whose target $y_t$ is held fixed during differentiation even
though it depends on $\Theta$ and $W$ through $Q(s_{t+1},\cdot)$. Thus, the update
direction $\delta_t\nabla_{\Theta,W}Q(s_t,a_t)$ omits the term
$-\gamma\nabla\max_{a'}Q(s_{t+1},a')$ (it is not $\nabla L_t$) and hence it is not the gradient of any fixed scalar function, so
semi-gradient temporal-difference learning is a stochastic approximation to a
fixed point of the projected Bellman optimality operator rather than a
descent method. The objective that is genuinely being optimized here is the
reinforcement-learning one, and because $r_t=0$ at every non-terminal step an
episode's return collapses to a single discounted terminal reward, making the
problem exactly one of optimal stopping,
\begin{equation}
J(\Theta,W)=\mathbb{E}_{(c_{1:T},\,y,\,t^\ast)\sim\mathcal{D}}
\bigl[\gamma^{\,\tau-1}R_\tau\bigr],
\qquad
R_\tau=\begin{cases}
-P_{\text{fa}} & y=0,\ \tau\le T,\\
B\,\beta^{\max(0,\tau-t^\ast)} & y=1,\ \tau\le T,\\
-P_{\text{miss}} & y=1,\ \text{never flagged},\\
0 & y=0,\ \text{never flagged},
\end{cases}
\label{eq:stopobj}
\end{equation}
where $\tau$ is the stopping time induced by the greedy policy. The one
component of the circuit update that is an exact gradient is the
anchor, $\lambda(\Theta-\Theta_0)=\nabla_\Theta\tfrac{\lambda}{2}
\lVert\Theta-\Theta_0\rVert_2^2$. We therefore report mean episode return
rather than a training loss throughout. Here $\mathcal{D}$ denotes the distribution over labelled conversations, each
draw being a triple $(c_{1:T},y,t^\ast)$ comprising a cue trajectory, its
label, and its critical turn, with the horizon $T$ itself part of the draw;
in our experiments $\mathcal{D}$ is the generative process of
Section~\ref{sec:testing}, and in deployment it would be the population of
monitored conversations. 

\subsubsection{Stability of the joint update mechanism}
It is well-known that training a feature map with a semi-gradient TD error is the deadly triad in
miniature \cite{van2018deep}. 
In particular, the TD targets are obtained by
bootstrapping from a function approximator whose underlying feature
representation is itself being updated during training. Consequently, both
the value estimates and their feature representation evolve simultaneously,
which can lead to unstable learning dynamics. In our experiments, such
instability was not only associated with slower convergence, in fact without
appropriate stabilization, the joint update learning occasionally collapsed to degenerate
policies that either never issued a flag or flagged every conversation in the early runs.
To mitigate this behavior, we employ the following three stabilization
mechanisms.

\begin{itemize}
\item \textbf{Warm-up.} The parameters of the quantum circuit are kept fixed during the first third
of the training epochs. This allows the Q-function to acquire a meaningful
initial estimate before the feature representation itself is allowed to
adapt.
\item \textbf{Anchoring.} 
We regularize the circuit parameters toward their cognitively motivated
initial values by adding
\[
-\lambda(\Theta-\Theta_0)
\]
to the circuit gradient, with $\lambda=0.02$. This regularization constrains
the learned encoder to remain interpretable while still allowing it to adapt to the
learning objective. The parameter $\lambda$ therefore controls a trade-off
between interpretability and adaptability. For example, with
$\lambda=0.06$, the model is unable to recover from the deliberately
mis-specified initialization considered in
Table~\ref{tab:pqc-ablation}.

\item \textbf{Per-episode, RMSProp-scaled steps.} 
The circuit gradients are accumulated and averaged over each episode,
followed by norm clipping and scaling using a running estimate of their
second moment. The circuit learning rate is set to
$\alpha_{\Theta}=0.003$. This procedure limits abrupt changes in the
learned feature representation caused by individual TD updates.
\end{itemize}



\subsection{Synthetic data generation}

In Figure \ref{fig:synthetic-data}, we describe our synthetic data generation process.

\begin{figure}[ht]
\centering
\includegraphics[width=1.0\linewidth]{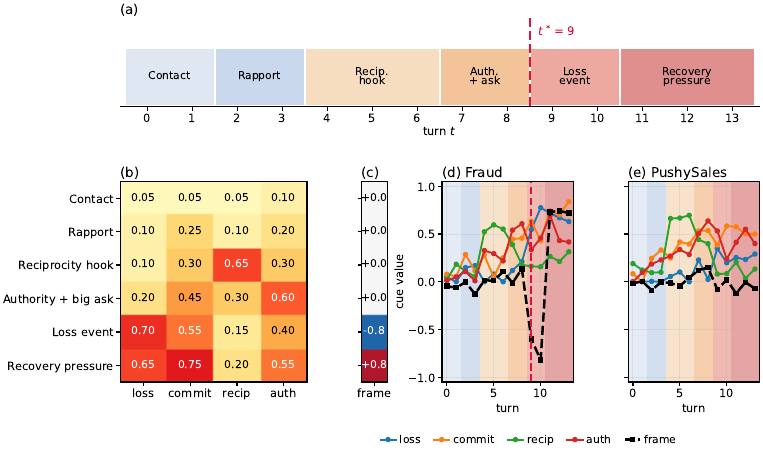}
\caption[Synthetic conversation generator.]{%
\textbf{The synthetic conversation generator.}
\textbf{(a)}~Each fraudulent conversation is divided into six stages by the
boundaries $\lfloor\mathrm{linspace}(0,T,7)\rfloor$, shown here for $T=14$;
the labelled critical turn $t^\ast$ is the first turn of the loss-event stage
(turn~9), and detection lateness is measured relative to it.
\textbf{(b)}~Each stage $s$ carries a fixed base level $b^{(s)}_i\in[0,1]$ for
the four bias channels. 
\textbf{(c)}~The frame channel $b^{(s)}_f\in[-1,1]$ is zero for the first four
stages and non-zero only in the last two, with opposite signs, so every
fraudulent conversation contains exactly one $R_Z\!\leftrightarrow\!R_Y$
transition. By the sign convention of Eq.~\eqref{eq:frame}, $c_{t,f}>0$ is loss
framing and drives $R_Z$ while $c_{t,f}<0$ is recovery framing and drives
$R_Y.$ 
\textbf{(d,e)}~Sampled realizations, drawn as
$c_{(t,i)}=\mathrm{clip}(b^{(s)}_i+\varepsilon,\ \ell_i,u_i)$ with
$\varepsilon\sim\mathcal{N}(0,\sigma^2)$, $\sigma=0.10$, independently per turn
and per channel, so there is within-stage variation but no autoregressive
structure. Shading marks the stage, the dashed red line marks $t^\ast$. 
}
\label{fig:synthetic-data}
\end{figure}


\end{document}